\documentclass[11pt]{article}

\usepackage[final]{acl}

\usepackage{times}
\usepackage{latexsym}
\usepackage{amsmath}
\usepackage{amssymb}
\usepackage{amsthm}
\usepackage{booktabs}
\usepackage{multirow}
\usepackage[utf8]{inputenc}
\usepackage{amssymb}
\usepackage{hyperref}
\usepackage{url}
\usepackage{amsmath, amsthm, amssymb}
\usepackage{booktabs}
\usepackage{multirow}
\usepackage{array}
\usepackage{xcolor, colortbl}

\usepackage{graphicx}

\newtheorem{corollary}{Corollary}

\usepackage{amsfonts}
\usepackage{amsmath}
\usepackage{booktabs}
\usepackage{multirow}
\usepackage{array}

\usepackage{algorithm}
\usepackage{algorithmic}
\usepackage{amsfonts}
\usepackage{amsmath}
\usepackage{booktabs}
\usepackage{multirow}
\usepackage{array}

\usepackage[T1]{fontenc}

\usepackage[utf8]{inputenc}

\usepackage{microtype}

\usepackage{inconsolata}

\usepackage{graphicx}

\title{Beyond the Sequentiality Bottleneck: A Hybrid Diffusion-Autoregressive Framework for Efficient Multimodal Reasoning}

\author{%
\parbox{\linewidth}{\centering
  \textbf{Wei Cai}\textsuperscript{1,2},\
  \textbf{Jian Zhao}\textsuperscript{2},\
  \textbf{Yuchen Yuan}\textsuperscript{2},\
  \textbf{Xuelong Li}\textsuperscript{2}\thanks{Corresponding authors.} \\[3pt]
  \normalfont
  \textsuperscript{1}Peking University \\
  \textsuperscript{2}Institute of Artificial Intelligence (TeleAI), China Telecom \\
}}
\usepackage{amsthm}
\begin{document}
\maketitle

\begin{abstract}

Autoregressive models (ARMs) face a "Sequentiality Bottleneck" in reasoning, while diffusion language models (DLMs) offer parallelism but lack semantic coherence. To addresses both limitations, we propose \textbf{Co}llaborative \textbf{R}easoning \textbf{A}symmetric \textbf{L}earning (CoRAL), a hybrid framework via an Asymmetric Planner-Executor (APE) architecture: a DLM Planner generates K diverse reasoning trajectories in parallel, while a large ARM Executor selects the optimal plan and synthesizes the final answer. To train the Planner effectively, we introduce Reward-Guided Denoising for CoT (RDC), which integrates outcome-based rewards into the diffusion process to steer it toward goal-directed reasoning. Experiments on vision-language benchmarks show that CoRAL significantly outperforms prior diffusion-based methods and achieves competitive performance against strong ARMs, achieving both superior efficiency and quality.
\end{abstract}

\section{Introduction}
\label{sec:intro}

Recently, autoregressive language models (ARMs)~\cite{bengio2003} have become the dominant paradigm in artificial intelligence, achieving remarkable performance across a wide range of applications~\cite{achiam2023gpt,team2023gemini}. These models excel in reasoning tasks through structured approaches such as step-by-step reasoning~\cite{kojima2022large,cobbe2021training}, chain-of-thought prompting~\cite{wei2022chain}, and plan-and-solve strategies~\cite{wang2023plan}.

Meanwhile, diffusion language models (DLMs)~\cite{yu2025discrete,li2025survey} have received increasing attention, and have recently emerged as a promising alternative to the dominant ARM paradigm, especially in the domain of multimodal large language models (MLLMs)~\cite{liu2023visual,bai2023qwen,zhu2023minigpt,cai2025safe,li2023blip,cai2026reasoning,team2024gemini,cai2025SRPO}, as the rich cross-modal inputs make the sequential ARMs more susceptible to hallucinations. Furthermore, processing complex information often requires expanding the token inference space to achieve higher reasoning accuracy, which exponentially increases the token consumption, and exacerbates the already-high computational and financial costs of large generative models.

\begin{figure*}[htbp]
\centering
\includegraphics[width=\textwidth]{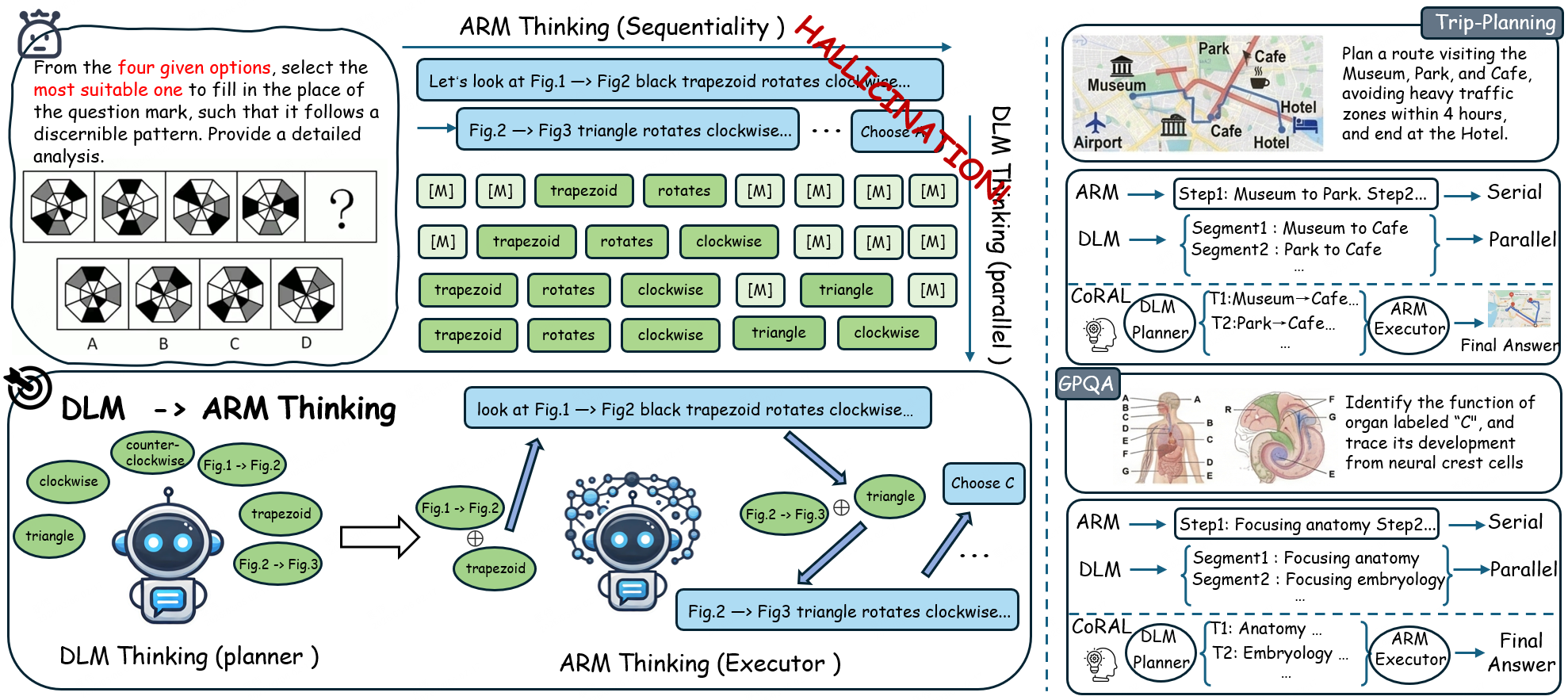}
\caption{Comparison of thinking in DLM, ARM, and DLM $\rightarrow$ ARM, DLMs: Parallel thinking, ARM: Sequential thinking, DLM $\rightarrow$ ARM (CoRAL): The DLM first plans in parallel, the ARM evaluates and executes}
\label{fig1}
\end{figure*}

Unlike ARMs which generate text token-by-token in a fixed unidirectional manner, DLMs conceptualize generation as an iterative denoising process. This approach enables bidirectional context modeling, granting DLMs greater flexibility in controlling the generation process and a theoretical potential for massive parallelization and significant gains in inference efficiency~\cite{nie2025large,you2025llada,li2025lavida}. Specifically, DLMs can parallelize the denoising process and generate multiple candidate lines of thought simultaneously. However, the quality of the generated text, as measured by standard metrics like perplexity, still lags behind that of ARMs~\cite{arriola2025block,gulrajani2023likelihood}, which hinders their practical applications. 

On the other hand, ARMs are more costly for the generation process due to their token-by-token manner, but are relatively efficient in processing inputs, \textit{i.e.} a single forward pass is sufficient to process and store the input, regardless of its length. Combined with their strong semantic understanding, they are well-suited to serve as executors for evaluating and outputting intermediate reasoning steps. 

As a result, ARMs and DLMs demonstrate complementary strengths: ARMs generates fluent, coherent, and human-understandable sequences through ``next-token prediction''~\cite{feng2025theoretical}, whereas DLMs support flexible token generation, a feature that recent studies have shown to be particularly advantageous for planning tasks~\cite{ye2023diffusion,zhang2023planner}, self-correction~\cite{hoogeboom2021argmax} and computational efficiency~\cite{lou2023discrete}.

In order to harness these complementary strengths and overcome their respective limitations, we propose the \textbf{Co}llaborative \textbf{R}easoning \textbf{A}symmetric \textbf{L}earning (CoRAL) framework, a novel hybrid architecture that pioneerly couples the parallel planning capabilities of DLMs with the high-fidelity causal synthesis of ARMs. Our approach is grounded in a formal analysis of these two paradigms, in which we identify a fundamental ``Sequentiality Bottleneck'' in standard autoregressive reasoning. We prove that a DLM-style parallel refinement paradigm is not only more efficient, but also strictly more expressive under logarithmic constraints.

More specifically, the CoRAL framework functions through an Asymmetric Planner-Executor (APE) architecture. We first introduce a novel Reward-Guided Denoising for COT (RDC) objective to train a DLM as an efficient planner, which takes advantage of its parallel nature to generate $K$ diverse reasoning trajectories simultaneously. These trajectory proposals are then processed by a large ARM executor, which performs an evaluation pass to select the optimal plan and synthesize a fluent, coherent, and correct final answer.

Our main contributions are as follows:
\begin{itemize}
    \item \textbf{Theoretical Foundation:} We identify the ``Sequentiality Bottleneck'' of ARMs and prove that parallel paradigms are strictly more expressive under logarithmic resource constraints, motivating the CoRAL design.

    \item \textbf{CoRAL Framework (Fig.2:} Built on an Asymmetric Planner-Executor (APE) architecture, CoRAL uses a DLM Planner to generate $K$ diverse reasoning trajectories in parallel, and a large ARM Executor to select the optimal plan and synthesize the final answer.

    \item \textbf{RDC Training Objective:} To make the APE Planner goal-directed, we introduce Reward-Guided Denoising for CoT (RDC), integrating outcome-based rewards into diffusion denoising to steer the Planner toward high-quality reasoning trajectories.

    \item \textbf{Empirical Validation:} CoRAL significantly outperforms existing DLMs and rivals strong ARMs in reasoning quality and inference efficiency across vision-language benchmarks.
\end{itemize}

\section{Related work}
\textbf{Diffusion for language.} Diffusion models~\cite{austin2021structured,lou2023discrete,sun2022score,sahoo2024simple} are a class of generative models designed for discrete data such as text. Unlike image diffusion model which corrupt data by adding Gaussian noise to a standard Gaussian prior—text diffusion models typically degrade the quality of semantic content by replacing original tokens. Early approaches~\cite{austin2021structured} employed discrete Markov chains, which progressively apply a transition matrix to the input, corrupting it towards a uniform distribution or an absorbing state (e.g., a [MASK] token). In recent years, ``mask-and-predict'' style diffusion models have achieved significant empirical results. For example, LLaDA~\cite{nie2025large} generates sentences starting from a fully masked sequence and progressively unmasks the tokens with the highest confidence, achieving performance comparable to autoregressive large language models; similarly, Dream~\cite{ye2025dream} has also achieved excellent results by initializing its parameters from a pretrained autoregressive model.

\begin{figure*}[t]
\centering
\includegraphics[width=\textwidth]{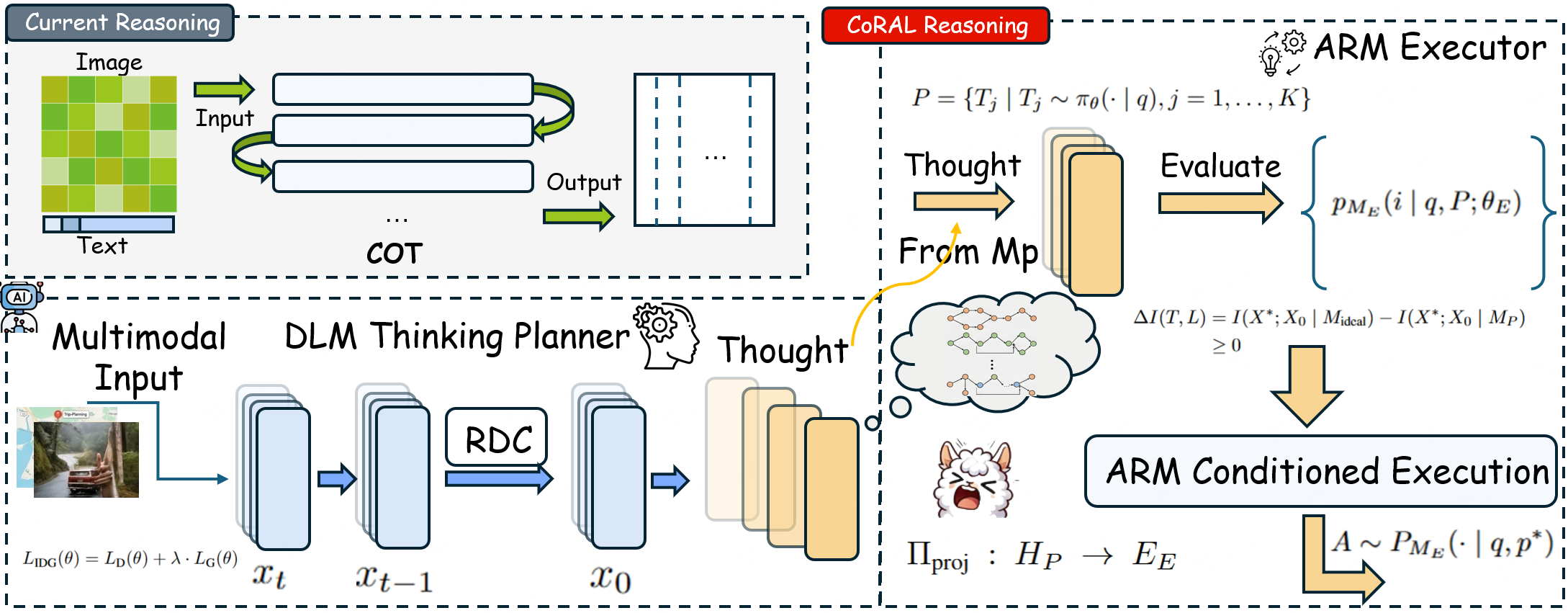}
\caption{An overview of our CoRAL framework.}
\label{fig4}
\end{figure*}

\textbf{Autoregressive model.} Autoregressive (AR) language models have demonstrated high reasoning accuracy, particularly when employing Chain-of-Thought (CoT) prompting and its variants, such as zero-shot CoT and self-consistency~\cite{cai2026visual,wei2022chain,cai2025SRPO}. These methods generate multi-step, token-by-token reasoning traces, which have been shown to reliably improve accuracy on arithmetic, commonsense, and symbolic tasks. However, the sequential nature of AR decoding makes longer CoT trajectories computationally expensive at inference time and can lead to verbosity or drift in long-horizon problems~\cite{chowdhery2023palm}. This has motivated the exploration of non-autoregressive mechanisms to decouple the ``cost of thought'' from the output length.

\section{Method}

\subsection{Reward-Guided Denoising for Chain-of-Thought}

Standard autoregressive reasoning suffers from a Sequentiality Bottleneck: generating a COT token-by-token precludes
parallel exploration of the solution space, making inference both computationally costly and susceptible to error accumulation over long horizons (see Appendix A.1 for a formal treatment). DLMs, with their inherent capacity for parallel token refinement, offer a natural remedy. However, a generic pre-trained DLM is optimized for distributional coverage rather than goal-directed reasoning, and therefore tends to produce trajectories that are fluent but not necessarily correct. To close this gap, we introduce RDC, which integrates outcome-based rewards directly into the diffusion training process to steer the planner toward high-quality, verifiably correct reasoning trajectories.

We conceptualize the reverse diffusion process as the generation of a stochastic CoT. Each denoising step, from a noisy state $x_t$ to a more refined state $x_{t-1}$, constitutes a single inferential ``thought'' in the chain. RDC aims to steer this chain generation process towards trajectories that culminate in correct answers.

\textbf{The Denoising Transition Operator as a Thought Step.}
The core of our DLM planner is the learned reverse transition operator $p_\theta(x_{t-1} | x_t, q)$, which defines the probability of transitioning to state $x_{t-1}$ given the current state $x_t$ and the prompt $q$. We posit that this operator represents a single, elementary step in the model's reasoning process. The entire CoT trajectory $\tau = (x_N, \dots, x_0)$ is thus sampled autoregressively over the diffusion timesteps:

\begin{equation}
\tau \sim \prod_{t=N}^{1} p_\theta(x_{t-1} | x_t, q)
\end{equation}

Standard diffusion training optimizes $p_\theta$ to maximize data log-likelihood, which is insufficient for complex reasoning tasks that require verifiable correctness.

\textbf{The Integrated Denoising-Guidance (IDG) Objective.}
Our key insight is to modulate the standard denoising objective with a terminal reward signal. Let $R(\tau)$ be a reward function that evaluates the correctness of the final state $x_0$ (e.g., $R(\tau) = 1$ if correct, $0$ otherwise). We integrate this reward directly into the diffusion framework by defining an Integrated Denoising-Guidance (IDG) loss function:

\begin{equation}
L_{\text{IDG}}(\theta) = L_{\text{D}}(\theta) + \lambda \cdot L_{\text{G}}(\theta)
\end{equation}

\textbf{$L_{\text{D}}$}: This is the conventional negative log-likelihood loss for diffusion models, which anchors the model to the data distribution and ensures stable training. For a discrete model, this can be the cross-entropy loss for predicting the clean $x_0$ from a noisy $x_t$.
\begin{equation}
L_{\text{D}} = -\mathbb{E}_{t, x_0, x_t}[\log p_\theta(x_0 | x_t)]
\label{eq:16}
\end{equation}

\textbf{$L_{\text{G}}$}: This term steers the generative process. It is formulated as the reward-weighted negative log-likelihood of the entire thought trajectory. Using a baseline $b$ for variance reduction, this loss is a policy gradient-style term:
\begin{equation}
L_{\text{G}}(\theta) = -\mathbb{E}_{\tau \sim p_\theta} 
    \bigl[ (R(\tau) - b) \cdot \log p_\theta(\tau | q) \bigr]
\label{eq:L_Guidance_1}
\end{equation}

By substituting Eq.~\ref{eq:16}, this becomes:

\begin{equation}
\begin{split}
L_{\text{G}}(\theta) &= -\mathbb{E}_{\tau \sim p_\theta} 
    \Bigl[ (R(\tau) - b) \\
    &\quad \cdot \sum_{t=N}^{1} \log p_\theta(x_{t-1} | x_t, q) \Bigr]
\end{split}
\label{eq:L_Guidance}
\end{equation}
This integrated objective (Eq.11) elegantly combines a standard diffusion objective with a guidance term, ensuring that every reasoning step $p_\theta(x_{t-1} | x_t)$ is reinforced based on its contribution to the final correct CoT.

\subsection{The CoRAL Inference Architecture: Asymmetric Planner-Executor Collaboration}

\paragraph{Architecture Overview.}
Figure.2 provides an end-to-end view of the CoRAL framework. 
The pipeline proceeds in two stages. 
In the \textit{planning stage} (left of Figure.2), the multimodal input is fed into the DLM Planner $M_P$, which is trained via the RDC objective to generate a set of $K$ parallel reasoning trajectories $\mathcal{P} = \{p_1, \ldots, p_K\}$.
In the \textit{execution stage} (right of Figure.2), the ARM Executor $M_E$ receives the original query alongside $\mathcal{P}$---either as surface-level text or via a learned projection $\Pi_{\text{proj}}: H_P \to E_E$---evaluates all proposals in a single forward pass, selects the optimal plan $p^*$, and generates the final answer $A$.

The Generative Trajectory Optimization framework detailed in \S3.2 yields a specialized Parallel Thought Planner. However, this planner, while efficient at exploring the solution space, may lack the linguistic fluency and causal reasoning fidelity of large autoregressive models. To address this, the CoRAL framework embeds this RDC-trained planner within an Asymmetric Planner-Executor (APE) architecture, which consist of:

\begin{itemize}
 \item \textbf{The Planner ($M_P$):} A DLM trained via RDC. Its sole responsibility is to receive a problem query $q$ and leverage its parallel refinement capabilities to generate a set of $K$ diverse, high-potential reasoning trajectories, or ``plans,`` denoted $P = \{p_1, p_2, \dots, p_K\}$.

 \item \textbf{The Executor ($M_E$):} A large autoregressive model (ARM). It is tasked with receiving the set of plans $P$ and the original query $q$, evaluating the proposals, and synthesizing the final, coherent, and correct answer $A$.
\end{itemize}

The synergy of this architecture is highly dependent on the computational and information-theoretic properties of each component and the interface between them.

\subsection*{Phase 1: Parallel Trajectory Proposal by the Planner}

The Planner's primary role is to efficiently explore a vast solution space, breaking the sequential bottleneck inherent to autoregressive generation.

\subsubsection*{Computational Amortization}
The key advantage of the DLM planner lies in its ability to generate $K$ distinct proposals in parallel with sub-linear time complexity scaling. We model this amortization effect with a Parallelism Saturation Exponent, $\beta$ (where $0 \leq \beta < 1$). The wall-clock time $\text{Time}_P$ to generate $K$ proposals is thus:
\begin{equation}
\text{Time}_P(K) = O(K^\beta \cdot T \cdot L)
\label{eq:time_p}
\end{equation}
where $T$ is the number of refinement steps and $L$ is the sequence length. This allows the Planner to efficiently generate the proposal set $P$:

\begin{equation}
 P = \{ T_j \mid T_j \sim \pi_\theta(\cdot \mid q), j=1,\dots,K \}
\end{equation}

where $\pi_\theta$ is the RDC-optimized policy.

\subsection*{Phase 2: Causal Evaluation and Execution by the Executor}

The Executor's role is not merely answer selection but to perform causal re-integration, which compensates the limitations of the parallel generation process.

The parallel decoding in the DLM planner suffers from an causal information deficit $\Delta I$. This deficit arises from the independence assumption made across tokens at each refinement step, leading to a loss of mutual information compared to an ideal, fully-dependent generation process:

\begin{equation}
\begin{split}
\Delta I(T, L) &= I(X^*; X_0 \mid M_{\text{ideal}}) \\
               &- I(X^*; X_0 \mid M_P) \geq 0
\end{split}
\label{eq:delta_I}
\end{equation}
This information deficit establishes a theoretical lower bound on the error rate of the Planner's proposals. The Executor's evaluation phase is designed to mitigate this deficit by re-contextualizing the acausal proposals within its own high-fidelity causal world model.

\textbf{Step A: Unified Evaluation and Selection Protocol.}
The Executor first performs a single, computationally efficient evaluation pass. Given the query $q$ and the full proposal set $P$, it is tasked with computing the conditional probability distribution over the proposal indices $i \in \{1, \dots, K\}$. We formulate this as:
\begin{equation}
p_{M_E}(i \mid q, P; \theta_E)
\end{equation}
where $\theta_E$ are the parameters of the Executor model. The optimal proposal $p^*$ is selected via an $\operatorname{argmax}$ operation:
\begin{align}
i^* &= \operatorname*{argmax}_i  p_{M_E}(i \mid q, P; \theta_E) \\
p^* &= P_{i^*}
\end{align}
This single-pass adopts the parallel processing of \texttt{prefill} in LLM inference, making the cost of evaluating $K$ proposals nearly equivalent to that of evaluating one.

\textbf{Step B: Conditioned Execution.}
After the selection is complete, the Executor is then instantiated for the final synthesis task. It generates the final answer $A$ by conditioning on both the original problem query $q$ and the chosen high-level plan $p^*$:
\begin{equation}
A \sim P_{M_E}( \cdot \mid q, p^*)
\end{equation}
This two-step process ensures that the final generation is guided by the most promising reasoning trajectory identified from the diverse set proposed by the Planner.

\begin{table*}[t]
\centering
\caption{Comparison of different models on various benchmarks}
\label{tab:model_comparison}
\setlength{\tabcolsep}{3pt}
\small
\begin{tabular*}{\textwidth}{@{\extracolsep{\fill}}l|cc|ccc|c|c|cc@{}}
\toprule
\multirow{2}{*}{Model} & \multicolumn{2}{c|}{Trip-Planning} & \multicolumn{3}{c|}{CapMAS} & CapArena & DetailCaps-4870 & \multicolumn{2}{c}{GPQA} \\
\cmidrule(lr){2-3} \cmidrule(lr){4-6} \cmidrule(lr){9-10}
 & Acc & ThroughPut & CLAIR & Coverage & Factuality & CapArena-Auto & CAPTURE & Acc & ThroughPut \\
\midrule
\multicolumn{10}{l}{\textit{AR model}} \\
LLaVA-1.5-7B & 0.05 & 0.24 & 62.20 & 34.50 & 52.60 & -93.00 & 51.05 & 0.32 & 0.15 \\
InternVL3-8B & 0.10 & 0.23 & 78.36 & 52.67 & 78.75 & -29.83 & 57.80 & 0.33 & 0.29 \\
VILA-1.5-7B & 0.07 & 0.19 & 68.47 & 42.52 & 58.69 & -49.75 & 55.30 & 0.32 & 0.19 \\
GLM-4v-9B & 0.11 & 0.25 & \underline{80.51} & 56.37 & 80.28 & -19.71 & 57.50 & 0.34 & 0.32 \\
MiniCPM-V-2.6 & 0.05 & 0.18 & 63.25 & 35.72 & 55.68 & -92.52 & 50.25 & 0.31 & 0.15 \\
Qwen2.5-VL-7B & 0.12 & 0.25 & 80.48 & \underline{57.32} & \underline{82.73} & \underline{-16.83} & \underline{60.61} & 0.34 & 0.38 \\
\midrule
\multicolumn{10}{l}{\textit{DLM model}} \\
MMaDA & 0.16 & 0.35 & 35.43 & 14.32 & 57.96 & -97.00 & 19.54 & 0.34 & 0.40 \\
FUDOKI & 0.19 & 0.36 & 51.92 & 39.18 & 46.01 & -98.82 & 57.95 & 0.34 & 0.45 \\
LaViDa & \underline{0.23} & \underline{0.38} & 56.21 & 44.17 & 53.58 & -90.00 & 57.26 & 0.35 & 0.46 \\
LLaDA-V & 0.21 & 0.37 & 65.54 & 49.23 & 61.05 & -77.16 & 59.63 & \underline{0.36} & \underline{0.47} \\
\rowcolor{gray!15}
CoRAL (ours) & \textbf{0.31}  & \textbf{0.42} & \textbf{86.75} & \textbf{58.36} & \textbf{83.59} & \textbf{-11.30} & \textbf{63.88} & \textbf{0.39} & \textbf{0.58} \\
\bottomrule
\end{tabular*}
\end{table*}

\subsection*{End-to-End Pipeline}

The input to the overall workflow is a reasoning problem. In the $M_P \to M_E$ architecture, the problem is first processed by the Planner ($M_P$), which generates $K$ textual plans. Subsequently, these plans can be transmitted to the Executor ($M_E$) through two distinct channels. The first is to append a plan as surface-level text to the original problem description. Alternatively, the Planner's final latent representation $z_i$ for each proposal is mapped into the Executor's embedding space $E_E$ via a small, learnable projection network $\Pi_{\text{proj}}: H_P \to E_E$. This allows for richer information transfer, enabling the Executor to condition on the Planner's ``thoughts'' before they are translated into potentially imperfect language. The Executor, instantiated as an ARM, then generates the final answer by conditioning on both the original problem and the provided plan.

\section{Experiment}
\subsection{Experiment Settings}

\textbf{Training Setup.} 
Our model is built upon the existing LLaDA-V framework. The training data comprises description datasets from LLaVA-1.5~\cite{liu2023visual}, ShareGPT4V~\cite{chen2024sharegpt4v}, and ViCrit~\cite{wang2025vicrit}. During the generation process, the model takes a masked sequence of length 256 as its initial input and completes the generation through a 256-step reverse diffusion process. We adopt block-level unmasking from LLaDA-V, partitioning the sequence into blocks of length 8. The unmasking policy module and the main body of LLaDA-V are jointly trained on 16 A100 GPUs. We use the AdamW optimizer with a learning rate of $5 \times 10^{-6}$ and set $(\beta_1, \beta_2)=(0.9, 0.999)$. The KL penalty is not applied by default, as it is found to yield negligible benefits in our experiments.

To enable the collaboration (DLM$\to$ARM), we train a projection layer, which maps the hidden states of the DLM into the embedding space of the ARM. During the training of this projection layer, the weights of both the ARM and the DLM are kept frozen. A cross-entropy loss function is adopted; conditioned on the latent space input, the predictions generated by the ARM are compared against the corresponding ground-truth tokens from the benchmark set, and the model is trained by minimizing the difference between them. In this experimental setup, we use LLaDa-V as the DLM Planner and Qwen2.5VL-7B as the ARM Executor.

\textbf{Benchmarks and Evaluation Setup.} 
We evaluate our model on three image description benchmarks: \textbf{CapMAS} \cite{lee2024toward}, which uses GPT-4o to assess three metrics: CLAIR for overall description quality, Coverage for completeness, and Factuality for content accuracy; \textbf{CapArena} \cite{cheng2025caparena}, which employs pairwise comparison where GPT-4o contrasts the test model's output against three baseline models, calculating a final score based on win rate; and \textbf{DetailCaps-4870} \cite{dong2024benchmarking}, which uses the CAPTURE metric to score generated descriptions by comparing their scene graphs with those of the ground-truth.

In addition, to assess the model's reasoning capabilities, we also evaluate it on \textbf{GPQA}~\cite{rein2024gpqa} and \textbf{Trip-Planning}~\cite{zheng2024natural}. GPQA consists of PhD-level multiple-choice questions from biology, physics, and chemistry, aiming the retrieval of obscure factual knowledge, multi-step scientific reasoning, and the ability to discern expert-designed distractors. Trip-Planning requires the formulation of a unique, valid multi-city itinerary under specific flight connection and daylight duration constraints, aiming the ability for implicit constraint propagation and cross-domain temporal reasoning.

We compare CoRAL against several vision-language diffusion models, including \texttt{LLaDA-V}~\cite{you2025llada}, \texttt{LaViDa}~\cite{li2025lavida},\texttt{MMaDA}~\cite{yang2025mmada}, and \texttt{FUDOKI}~\cite{wang2025fudoki}. For reference, we also include results from representative AR based MLLMs, which include the \texttt{Qwen2.5-VL} series~\cite{bai2025qwen2}, \texttt{InternVL2} series~\cite{chen2024internvl}, \texttt{GLM-4V}~\cite{glm2024chatglm}, \texttt{LLaVA} series~\cite{li2024llava}, \texttt{MiniCPM-v2.6}~\cite{yao2024minicpm}, and \texttt{VILA} series~\cite{lin2024vila}. During the inference stage, the maximum length for generated text is set to 128. An inference process of 128 steps corresponds to a speed of ``1 token/step'', while a 32-step process corresponds to a speed of ``4 tokens/step''.

\subsection{Main Result}
As shown in Table 1, our proposed CoRAL framework shows clear and consistent superiority across all benchmarks, outperforming SOTA ARMs and other diffusion-like models. In detailed image description, CoRAL exhibits outstanding performance. On the three metrics of CapMAS, CoRAL achieves high scores of 86.75 (CLAIR), 58.36 (Coverage), and 83.59 (Factuality), not only substantially surpassing the strongest DLM baseline but also outperforming top-tier AR models such as Qwen2.5-VL and GLM-4v. This indicates that, guided by the high-quality plans provided by the RDC-trained planner, the executor is capable of generating descriptions that are both comprehensive and factually accurate. Furthermore, CoRAL's exceptional performance on CapArena (-11.30) and DetailCaps-4870 (63.88) further solidifies its leading position in detailed image understanding.

\begin{figure}[t]
\centering
\includegraphics[width=\columnwidth]{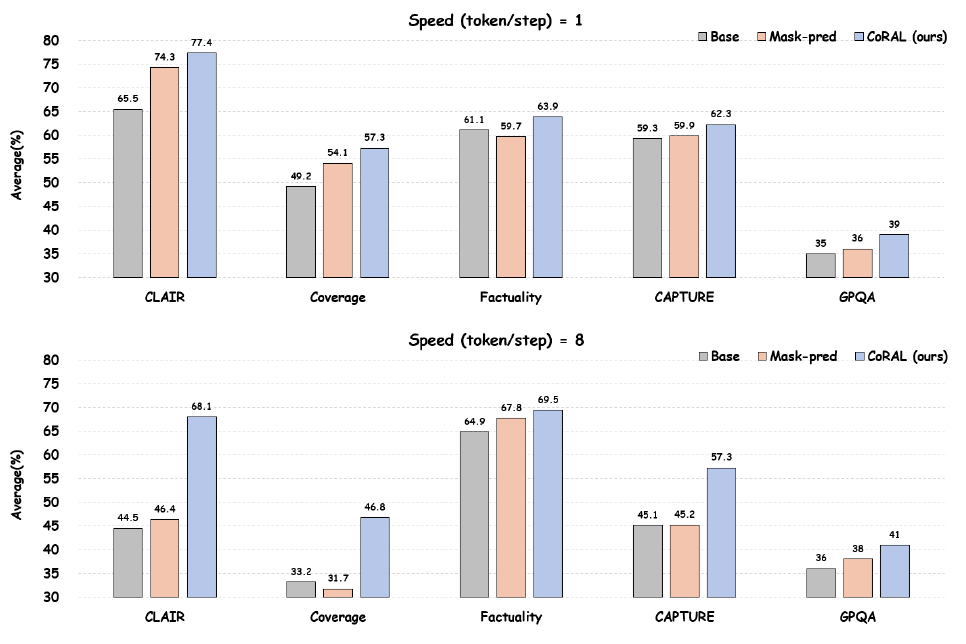}
\caption{Performance comparison across different inference steps on the benchmark datasets. ``Mask-pred training'' refers to training conducted on the same dataset with the traditional masked prediction objective}
\label{fig3}
\end{figure}

\begin{figure*}[t]
\centering
\includegraphics[width=\textwidth]{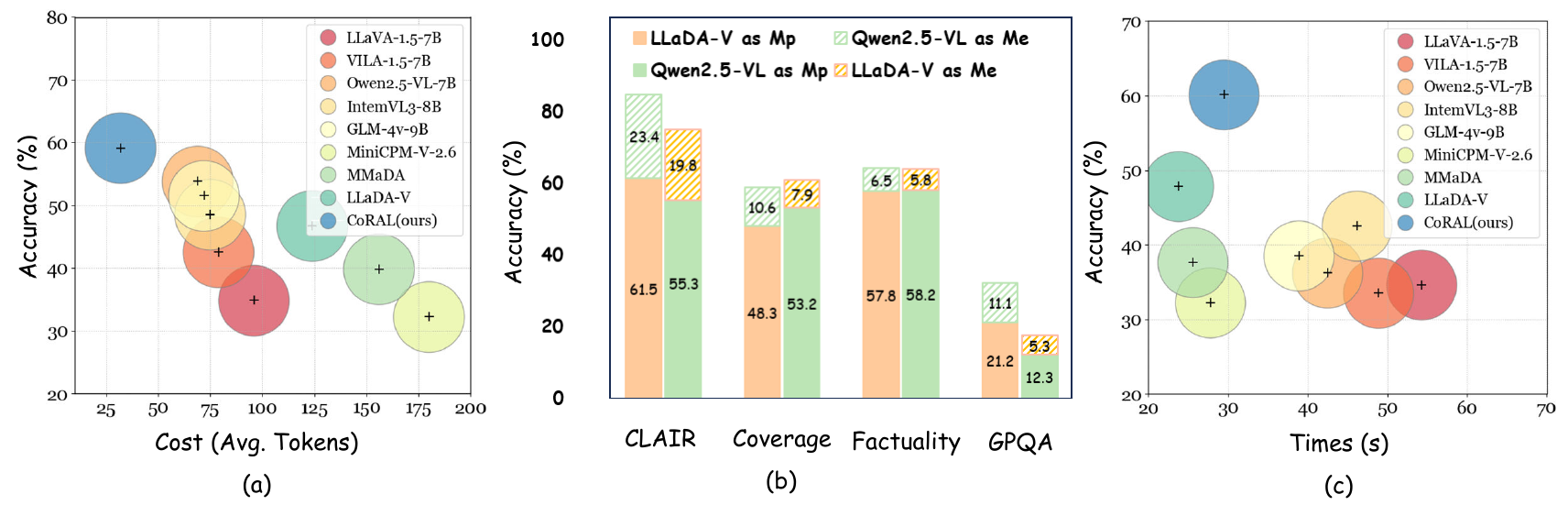}
\caption{(Left) The accuracy-cost trade-off for various models. (Center) Performance of the DLM and ARM when serving as the Planner ($M_p$) and Executor ($M_e$), respectively. (Right) The accuracy-runtime trade-off for various models.}
\label{fig4}
\end{figure*}

In terms of efficiency, CoRAL boasts the highest throughput (0.42 on Trip-Planning and 0.58 on GPQA), confirming the computational advantages of its parallel generation method. Moreover, CoRAL achieves SOTA performance on complex reasoning tasks. On the Trip-Planning task, which demands complex constraint satisfaction, CoRAL achieves an accuracy of 0.31, more than double that of the best-performing AR model, Qwen2.5-VL-7B (0.12), and significantly surpasses all other DLMs. Similarly, on the expert-level scientific reasoning benchmark GPQA, CoRAL attains an accuracy of 0.39, leading all other models. These results robustly validate our hypothesis that a parallel planning mechanism can effectively overcome the inherent ``sequential bottleneck'' of AR models in reasoning-intensive problems.

In Figure 3, we compare CoRAL with a model trained with the traditional masked prediction objective. They both use the same training set and base model. The results show that our model consistently outperforms the masked-prediction model in all of the inference steps. Moreover, as the generation speed increases (\textit{i.e.}, with fewer inference steps), our model's performance degrades less significantly, demonstrating better stability in parallel generation.

For example, on the CLAIR metric, when the generation speed is increased from 1 token/step to 8 tokens/step, the score of the masked-prediction model plummets from 74.3 to 46.4, whereas our model's score only decreases from 77.4 to 68.1. The trend is less pronounced for the Factuality and GPQA metrics, as the baseline model's score does not drop significantly with fewer inference steps. This phenomenon can be attributed to the metric's reliance on extracting ``verifiable information items''; as the baseline's output becomes increasingly incoherent, the number of extractable valid items diminishes, leading to an artificial stabilization of the ``accuracy'' score. On both the Coverage and CAPTURE metrics, our model again showcases more robust parallel generation capabilities; when the generation speed increases from 1 token/step to 8 tokens/step, the magnitude of the performance drop on both metrics is far smaller than that of the masked-prediction baseline.

\begin{table*}[t]
\centering
\small
\caption{Ablations of each component of our CoRAL}
\begin{tabular}{llccccccc}
\toprule
\multicolumn{2}{c}{} & & & & \textbf{Trip-Planning} & \textbf{CLAIR} & \textbf{Coverage} & \textbf{CAPTURE}\\
& & ARM & DLM & RDC & Acc & Avg & Avg & Avg \\
\midrule
& Qwen2.5-VL-7B (ARM-Only) & \checkmark & & & 0.12 & 80.48 & 57.32 & 60.61 \\
& LLaDA-V (DLM-Only)  & & \checkmark & & 0.21 & 65.54 & 49.23 & 59.63 \\
& \textbf{CoRAL} & \checkmark & \checkmark & \checkmark & \textbf{0.31} & \textbf{86.75} & \textbf{58.36} & \textbf{63.88} \\
& \quad $w/o$ RDC & \checkmark & \checkmark & & 0.25 & 82.36 & 57.83 & 62.27\\
 & \quad $w/o$ ARM & & \checkmark & \checkmark & 0.26 & 74.35 & 52.21 & 61.75\\
& \quad $w/o$ DLM & \checkmark & & \checkmark & 0.18 & 81.23 & 57.51 & 61.10\\
& Mask-pred & & & & 0.23 & 71.45 & 50.85 & 60.98\\
\bottomrule
\end{tabular}
\label{tab:ablation}
\end{table*}

\subsection{Further Analysis}
\textbf{Visualization of Performance-Efficiency Trade-offs. }Figure 4(a) illustrates the relationship between model accuracy and average token cost. It is clearly shown that CoRAL (represented by the blue circle) occupies the top-left, achieving the highest accuracy while maintaining extremely low token cost. This indicates CoRAL's exceptional capability to produce high-quality reasoning with minimal computational overhead. Similarly, Figure 3(c) elucidates the accuracy-runtime trade-off. Although CoRAL's runtime is slightly slower than the DLM-only baseline, it achieves the highest accuracy in a relatively short time, further substantiating the practical advantages of its parallel planning architecture in accelerating reasoning.

\textbf{Analysis of Planner-Executor Component Efficacy. }Figure 4(b) validates our Asymmetric Planner-Executor (APE) design by comparing role-swapped configurations of the DLM (LLaDA-V) and ARM (Qwen2.5-VL). Our CoRAL setup (DLM $\rightarrow$ ARM) dramatically outperforms the reverse (ARM $\rightarrow$ DLM) on overall quality, achieving a total CLAIR score of 84.9 versus 75.1. While the reverse configuration yields slightly higher Coverage, this merely reflects the ARM's bias toward exhaustive, sequential enumeration. As a planner, this verbose detailing harms structural coherence and complex reasoning. Factuality remains comparable across both setups, driven primarily by the foundational vision encoders. This confirms our core claim: DLMs are ideal for concise, parallel planning, leaving large ARMs to serve as high-fidelity causal executors.

\subsection{Ablation Study}

The core contributions of CoRAL lie in the design of its RDC module and the hybrid inference architecture. In this section, we evaluate the effectiveness of these designs, as shown in Table 2. It is evident that the integrated CoRAL, which combines both the ARM and the DLM, significantly outperforms the ARM-only and DLM-only variants. Besides, the ARM demonstrates strong performance on the CLAIR and Coverage benchmarks, indicating its proficiency in image description tasks.

When RDC is removed and the DLM planner is trained instead with a standard masked prediction objective, a significant performance degradation is observed across all metrics. For instance, on the complex reasoning task of Trip-Planning, the accuracy dropped from 0.31 to 0.25; on the CLAIR score for image description quality, the score decreased from 86.75 to 82.36. This shows that RDC is crucial to guide the planner in generating high-quality, goal-oriented reasoning trajectories.

We also evaluate a variant that relies solely on the RDC-trained DLM planner (w/o ARM). Although this model still performs better than the baseline DLM, it is substantially inferior to the full CoRAL framework. Its CLAIR score plummeted from 86.75 to 74.35, and its CAPTURE score fell from 63.88 to 61.75. This highlights the critical role of the executor in performing causal reintegration and structuring the final output.

Conversely, removing the DLM planner (w/o DLM) and forcing the ARM to operate under an alternative guidance mechanism severely impairs its reasoning capabilities. The accuracy on Trip-Planning drops drastically from 0.31 to 0.18, which indicates that although the ARM is a powerful executor, it cannot efficiently explore the vast solution space on its own. The DLM planner's ability to generate diverse, high-potential proposals in parallel is fundamental to breaking the sequential bottleneck and finding solutions to complex problems.

\section{Conclusion}
In this work, we introduce CoRAL, a hybrid diffusion-autoregressive framework designed to overcome the Sequentiality Bottleneck in multimodal reasoning. We first present a theoretical analysis validating the advantage of parallel paradigms, then propose RDC to train a DLM as an efficient reasoning planner. This planner is then unified with ARM in an APE architecture. Experiments show that CoRAL outperforms existing methods in both reasoning quality and computational efficiency, effectively combining the parallel exploration of DLMs with the causal precision of ARMs.

\section*{Limitations}
CoRAL relies on a two-stage pipeline coupling a DLM Planner with a large ARM Executor, which introduces additional system complexity and inference latency compared to single-model baselines. While CoRAL achieves strong performance across the evaluated benchmarks, several limitations remain. Our experiments are conducted on English-language datasets; the effectiveness of the framework on multilingual or low-resource language tasks has not been investigated. In addition, the current evaluation focuses on static image inputs, and extending CoRAL to video understanding or streaming multimodal inputs is left for future work. 

\bibliography{custom}

\newpage
\appendix
\section{Appendix}

\subsection{Theoretical Analysis}

To compare the efficiency between parallel and sequential generation paradigms, as represented by DLMs and ARMs, we formalize the two as distinct computational processes:

\begin{itemize}
    \item \textbf{Sequential Generation} represents the class of problems solvable by an ARM within $T$ sequential steps. We denote this as $C_{\text{SEQ}}(T)$.
    
    \item \textbf{Parallel Refinement ($C_{\text{PAR}}$)} represents the class of problems solvable by a DLM through $T$ parallel refinement steps with an auxiliary computational space of size $S$. We denote this as $C_{\text{PAR}}(T, S)$.
\end{itemize}


\textbf{Emulation of Sequential Processes via Parallel Refinement:} A sequential generation process can be emulated within the parallel refinement framework. This requires the parallel model to explicitly enforce a unidirectional causal flow, an imposed constraint to its architecture. The overhead of this constraint, which can be the extra computational space involved that leads to additional memory cost, is formalized as the Causality Imposition Cost $\Lambda_C$. The simulation is expressed as:
\begin{equation}
C_{\text{SEQ}}(T, S) \subseteq C_{\text{PAR}}(T, S + \Lambda_C(N, S))
\label{}
\end{equation}
where $\Lambda_C$ represents the additional auxiliary space required. The cost $\Lambda_C$ stems from the burden of simulating a masked, causal self-attention mechanism with a fully visible, unmasked one. To maintain the integrity of the causal chain, the model must expend extra resources to prevent information leakage from future steps. This results in a cost that scales quadratically with the total sequence length:
\begin{equation}
\Lambda_C(N, S) = O((N+S)^2)
\label{}
\end{equation}
This ``causality tax'' highlights a fundamental inefficiency when a natively parallel architecture is constrained to perform a strictly sequential task.

\textbf{Unfolding of Parallel Processes via Sequential Generation:}
In contrast, a sequential model can simulate a parallel refinement process as well. This is achieved through a state-space unfolding operator, $\Pi_U$, which linearizes the entire parallel state of the $C_{\text{PAR}}$ model into a token sequence at each refinement step. The total length of the required sequential trace, $T_{\text{SEQ}}$, is the cumulative sum of these unfolded representations:
\begin{align}
T_{\text{SEQ}} &= \Pi_U(T, L, N, S) \nonumber \\
&= T \times \text{Cost}_U(H_{\text{PAR}})
\end{align}
where $\text{Cost}_U$ represents the description length of a single parallel state $H_{\text{PAR}}$, which is a function of the parallel model's parameters, scaling linearly with its depth $L$ and state size $(N+S)$. The formal inclusion is therefore:
\begin{equation}
C_{\text{PAR}}(T, S) \subseteq C_{\text{SEQ}}( T \times O(L \times (N + S)) )
\label{}
\end{equation}
This simulation transforms $T$ parallel steps into a vastly longer sequential trace, demonstrating the inherent overhead of serializing a parallel computation.

\textbf{The Hierarchy of Simulation Overheads:}
The aforementioned asymmetric costs lead to a clear hierarchy. For a given number of steps $T$, a sequential process can be emulated by a parallel one with quadratic spatial overhead, and a parallel process can be emulated by a sequential one with multiplicative overhead on its temporal length. For $T \geq N$, this hierarchy simplifies to:
\begin{align}
C_{\text{SEQ}}(T) &\subseteq C_{\text{PAR}}(T, O(T^2)) \nonumber \\
&\subseteq C_{\text{SEQ}}(O(T^3))
\end{align}
This demonstrates that while inter-simulation between ARMs and DLMs is possible, the polynomial costs are substantial and asymmetric, exhibiting their significant difference in efficiency.

Given the polynomial nature of the simulation costs, the computational burden of DLMs with polynomial number of refinement steps remains within polynomial time ($\mathbf{P}$).
\begin{equation}
C_{\text{PAR}}(\text{poly}(N), \text{poly}(N)) \subseteq \mathbf{P}
\label{}
\end{equation}


The asymmetry in simulation costs above then lead to our central separation theorem.

\begin{corollary}[Strict Expressivity Separation under Logarithmic Constraints]
When computational resources are constrained to a sub-linear, logarithmic budget ($T = \log N$), a strict separation in expressivity emerges. The class of problems solvable by the parallel paradigm is a strict superset of that solvable by the sequential paradigm:
\begin{equation}
C_{\text{SEQ}}(\log N) \subset C_{\text{PAR}}(\log N, N)
\label{}
\end{equation}
\end{corollary}
This separation is a result of the Sequentiality Bottleneck inherent to the $C_{\text{SEQ}}$ paradigm, mainly due to its structural inability to exploit parallelism in problem decomposition. Concretely, the set difference $C_{\text{PAR}}(\log N, N) \setminus C_{\text{SEQ}}(\log N)$ is non-empty and contains problem classes with high degrees of inherent parallelism, such as $NC^1$-complete regular languages. This result validates that for certain classes of problems, the parallel refinement paradigm is fundamentally more powerful and efficient.

\subsection{Theoretical Justification}

Our formalization of the CoRAL framework, centered on the Planner-Executor architecture, is motivated by the need to analyze and harness the true expressive power of Diffusion Language Models (DLMs) for reasoning. While our definition of the Planner intentionally abstracts away from the exact analytical reverse process, it remains well-grounded in both empirical practice and theoretical considerations. In the following, we justify our modeling choices.

At an intuitive level, our Parallel Thought Planner ($M_P$) distills the essential structure of a practical DLM: a process that iteratively refines a reasoning trajectory by unmasking and instantiating symbols. In our idealization, this process is governed by a learned policy that strategically chooses which symbols to refine and when. This constitutes a principled departure from the analytical reverse process of the forward masking, which would mandate a naive Uniform Sampling Policy.

This departure is motivated by two critical factors:

\begin{enumerate}
    \item \textbf{Empirical Sub-optimality of the Uniform Policy:} Forcing a trained DLM to adhere to a uniform unmasking schedule is demonstrably suboptimal. It prevents the model from developing an intelligent reasoning order (e.g., solving easier sub-problems first) and can lead to uncontrolled error propagation. Our conceptualization of a dedicated Strategic Planner captures the mechanisms behind a wide array of empirically successful, context-aware unmasking strategies that consistently outperform random sampling, such as:
    \begin{itemize}
        \item \textbf{Confidence-based Unmasking:} Prioritizing the refinement of symbols for which the model has the highest prediction confidence.
        \item \textbf{Difficulty-based Scheduling:} Masking informative symbols longer so they are generated last with maximum context.
        \item \textbf{Structured and Learned Policies:} Employing structured schedules (e.g., blockwise-autoregressive) or entirely separate learned models to guide the unmasking process.
    \end{itemize}

    \item \textbf{Theoretical Inefficiency of the Uniform Policy:} From a theoretical standpoint, the Uniform Sampling policy is computationally inefficient under the factorized reverse process assumption. It requires a number of denoising steps that scales at least linearly with the sequence length to capture dependencies in even simple formal languages. Furthermore, the KL divergence to the true data distribution converges slowly under this policy. This theoretical lower bound negates the primary advantage of employing a DLM over an autoregressive model for parallel thought generation.
\end{enumerate}

Therefore, a Strategic Planning Policy---as implemented by our RDC training frameworks---is not merely an enhancement but a necessary component for an efficient and accurate Parallel Thought Planner. It provides an educated, context-aware mechanism for exploring the reasoning space, mitigating the weaknesses inherent in the factorized independence assumption.

To further ground our model in practice, we note that the conceptual separation of the Planner's policy into a ``structural'' component (where to unmask) and a ``semantic'' component (what to predict) can be unified. We show that any DLM implemented as a combination of these two policies can be instantiated as a single transformer that performs confidence-based (e.g., top-$k$) decoding. This demonstrates that our theoretical framework and its results are directly applicable to the most popular and effective implementations of diffusion models in practice.

\subsection{Functional Equivalence of Decomposed and Monolithic Policy Models}

Our theoretical framework conceptualizes the strategic planner as a composition of two policies. Here, we formally demonstrate that this decomposed representation is functionally equivalent to a single, monolithic transformer model that implements a confidence-based \texttt{top-k} selection mechanism. This establishes that our theoretical results are directly applicable to practical, unified model architectures.

\medskip
\noindent \textbf{Definition A.1 (Decomposed Policy Model).} A Decomposed Policy Model $\mathcal{M}_{\Pi}$ is a pair $(\pi_{\text{struct}}, \pi_{\text{sem}})$ where:
\begin{enumerate}
    \item $\pi_{\text{struct}}: \Sigma_m^* \to \{0,1\}^N$ is the \textbf{Structural Policy}, a deterministic function mapping a sequence $y$ to a binary vector indicating which of the $N$ positions to unmask.
    \item $\pi_{\text{sem}}: \Sigma_m^* \to (\mathcal{P}(\Sigma))^N$ is the \textbf{Semantic Policy}, a function mapping $y$ to a sequence of $N$ probability distributions over the vocabulary $\Sigma$ for infilling.
\end{enumerate}

\medskip
\noindent \textbf{Definition A.2 (Monolithic Top-k Model).} A Monolithic Top-k Model $\mathcal{M}_T$ is defined by a single transformer $T$ and an integer $k$. Let $\Lambda(y) \in \mathbb{R}^{N \times |\Sigma|}$ be the logit tensor produced by $T$ on input $y$. The model's behavior is defined by:
\begin{enumerate}
    \item An index selection set $U_k(y) = \operatorname{TopK}(\Lambda(y), k)$, where $\operatorname{TopK}$ returns the indices of the $k$ positions with the highest maximal logit values.
    \item An infilling distribution $p_T(v_n \mid y) = \operatorname{Softmax}(\Lambda(y)_n)$ for $n \in U_k(y)$.
\end{enumerate}

\medskip
\noindent \textbf{Theorem A.1 (Equivalence of Policy Models).} For any Decomposed Policy Model $\mathcal{M}_{\Pi} = (\pi_{\text{struct}}, \pi_{\text{sem}})$ subject to the cardinality constraint $\|\pi_{\text{struct}}(y)\|_0 = k$ (where both policies are computable by transformers), there exists an equivalent Monolithic Top-k Model $\mathcal{M}_T$ such that for any input $y \in \Sigma_m^*$:

\begin{subequations}
\begin{align}
    \operatorname{TopK}(\Lambda_T(y), k) &= \{n \mid \pi_{\text{struct}}(y)_n = 1\} \\
    p_T(v_n \mid y) &= \pi_{\text{sem}}(v_n \mid y)_n 
\end{align}
\end{subequations}

\noindent \textit{Proof.} The proof is constructive. Let the transformers implementing the policies of $\mathcal{M}_{\Pi}$ produce logit tensors $\Lambda_{\text{struct}}(y) \in \mathbb{R}^N$ (for the binary decision at each position) and $\Lambda_{\text{sem}}(y) \in \mathbb{R}^{N \times |\Sigma|}$ respectively. We construct the monolithic transformer $T$ to compute a single, fused logit tensor $\Lambda_T(y)$.

\begin{enumerate}
    \item \textbf{Construction of the Fused Logit Tensor:}
    The unified transformer $T$ is constructed to compute $\Lambda_{\text{struct}}(y)$ and $\Lambda_{\text{sem}}(y)$ in parallel internal states. The final output logit tensor $\Lambda_T(y)$ is defined by the following fusion operation for each position $n \in \{1,\dots,N\}$ and vocabulary item $v \in \Sigma$:
    \begin{equation}
        \Lambda_T(y)_{n,v} := \Lambda_{\text{sem}}(y)_{n,v} + \log(\pi_{\text{struct}}(y)_n) 
    \end{equation}
    We define this operation over the extended real numbers $\mathbb{R} \cup \{-\infty\}$, where $\log(1) = 0$ and $\log(0) = -\infty$. This operation is implementable within a transformer architecture through parallel computation followed by element-wise addition.

    \item \textbf{Verification of Equivalence:}
    \begin{itemize}
        \item \textbf{Structural Equivalence (1a):} Consider the $\operatorname{TopK}$ operator applied to $\Lambda_T(y)$. If $\pi_{\text{struct}}(y)_n = 0$, then $\log(\pi_{\text{struct}}(y)_n) = -\infty$. Consequently, all logit values $\Lambda_T(y)_{n,v}$ for that position $n$ become $-\infty$, effectively inhibiting selection. 
        If $\pi_{\text{struct}}(y)_n = 1$, then $\log(\pi_{\text{struct}}(y)_n) = 0$, and the logits $\Lambda_T(y)_{n,\cdot}$ are simply $\Lambda_{\text{sem}}(y)_{n,\cdot}$. Since $\pi_{\text{struct}}$ is constrained to select exactly $k$ indices (where logits are finite) and all other indices are $-\infty$, the $\operatorname{TopK}$ operator will precisely select the set of indices where $\pi_{\text{struct}}(y)_n = 1$.
        
        \item \textbf{Semantic Equivalence (1b):} For any index $n$ selected by the $\operatorname{TopK}$ operator, it must be the case that $\pi_{\text{struct}}(y)_n = 1$. According to Eq. (2), for these indices, the $\log$ term is zero. Therefore, $\Lambda_T(y)_{n,\cdot} = \Lambda_{\text{sem}}(y)_{n,\cdot}$. Applying the $\operatorname{Softmax}$ function to these identical logit vectors yields identical probability distributions. Thus, $p_T(v_n \mid y) = \pi_{\text{sem}}(v_n \mid y)_n$.
    \end{itemize}
\end{enumerate}

This construction demonstrates that the conceptual decomposition of the planner is mathematically equivalent to a single, unified transformer model under the cardinality constraint. \hfill $\square$

\subsection{The Computational Reducibility of Iterative Refinement to Monotonic Unmasking}

Our proposed Planner allows for Iterative Refinement, where the model is permitted to resample (i.e., edit) previously unmasked tokens. This departs from the analytical reverse process of standard diffusion, which strictly adheres to a Monotonic Unmasking schedule. While Iterative Refinement is empirically motivated by the need for self-correction, it is crucial to establish its theoretical relationship with the stricter Monotonic regime.

Here, we formally prove that the distinction is one of resource trade-offs rather than fundamental computability. Specifically, any Planner with unrestricted refinement capabilities can be simulated by a mask-dominated (monotonic) Planner, provided the latter is augmented with a memory workspace that scales linearly with the number of refinement steps.

\medskip
\noindent \textbf{Definition A.3 (Refining vs. Monotonic Planners).}
\begin{itemize}
    \item Let $\mathcal{M}_{\text{Ref}}$ be a \textbf{Refining Planner} operating on a sequence space $\Sigma^P$. At step $t$, it updates the state $x_t$ based on $x_{t-1}$, allowing $x_t[i] \neq x_{t-1}[i]$ for any index $i$.
    \item Let $\mathcal{M}_{\text{Mon}}$ be a \textbf{Monotonic Planner}. It operates under the constraint that if $x_{t-1}[i] \neq \texttt{mask}$, then $x_t[i] = x_{t-1}[i]$.
\end{itemize}

\medskip
\noindent \textbf{Theorem A.2 (Simulation of Refinement via Space-Time Expansion).}
The computational process of a Refining Planner $\mathcal{M}_{\text{Ref}}$ over $T$ steps with output space $P$ is reducible to that of a Monotonic Planner $\mathcal{M}_{\text{Mon}}$ operating over $T$ steps with an augmented output space of $T \times P$. Formally:
\begin{equation}
    \mathcal{C}_{\text{Ref}}(T, P) \subseteq \mathcal{C}_{\text{Mon}}(T, T \times P)
\end{equation}

\noindent \textit{Proof.} The proof constructs a simulation mapping $\Psi$ that unfolds the temporal evolution of $\mathcal{M}_{\text{Ref}}$'s state into the spatial dimensions of $\mathcal{M}_{\text{Mon}}$.

\begin{enumerate}
    \item \textbf{State-Space Augmentation:} We allocate a ``history tape'' of $T$ blocks in $\mathcal{M}_{\text{Mon}}$'s output space, denoted as $B_1, B_2, \dots, B_T$, where each block $B_k$ has size $P$. The total space is $T \times P$.

    \item \textbf{Planner Adaptation (The Selection Policy):} The structural policy of $\mathcal{M}_{\text{Mon}}$ is defined to be deterministic and sequential regarding blocks. At step $t$, the planner selects the entire block $B_t$ to be unmasked, while leaving $B_{t+1}, \dots, B_T$ masked and $B_1, \dots, B_{t-1}$ fixed. This effectively turns the $t$-th step of $\mathcal{M}_{\text{Ref}}$ into the $t$-th spatial write operation of $\mathcal{M}_{\text{Mon}}$.
    
    \item \textbf{Predictor Adaptation (The Transition Operator):} The predictor of $\mathcal{M}_{\text{Mon}}$ is configured to simulate the transition operator of $\mathcal{M}_{\text{Ref}}$. Specifically, to generate the content of block $B_t$ (which represents the state $x_t$ of $\mathcal{M}_{\text{Ref}}$), the attention mechanism must strictly attend to block $B_{t-1}$ (which represents $x_{t-1}$).
    
    Let $f_\theta$ be the transition function of $\mathcal{M}_{\text{Ref}}$ such that $x_t = f_\theta(x_{t-1})$. We construct $\mathcal{M}_{\text{Mon}}$ such that:
    \begin{equation}
        B_t = f_\theta(B_{t-1})
    \end{equation}
    This is achievable via standard transformer architectures by imposing a \textbf{block-diagonal attention mask} (restricting the receptive field of $B_t$ to $B_{t-1}$) and manipulating positional encodings to realign the semantic positions of the history tape to the current step's logical inputs.
\end{enumerate}

\noindent \textbf{Conclusion:} By treating the sequence of states $\{x_0, \dots, x_T\}$ in the Refining Planner as a single concatenated sequence $[x_0, \dots, x_T]$ in the Monotonic Planner, $\mathcal{M}_{\text{Mon}}$ achieves exact simulation. This result confirms that the choice to allow resampling in our framework serves to improve \textbf{spatial efficiency} (reducing memory from $\mathcal{O}(TP)$ to $\mathcal{O}(P)$) without sacrificing the theoretical universality of the generation process. \hfill $\square$

\textbf{Analysis of Planner-Executor Configurations.} To validate our hybrid design, Table 3 compares various Planner-Executor configurations. We find that our Asymmetric Planner-Executor (APE) relies heavily on proper role assignment. Deploying a DLM as the Planner and an ARM as the Executor (DLM $\rightarrow$ ARM) yields significantly better performance (0.25 on Trip-Planning) than the reverse ARM $\rightarrow$ DLM configuration (0.20). This confirms that DLMs excel at parallel exploration, while ARMs are superior for final causal synthesis. Furthermore, upgrading the module interaction from discrete text to latent embeddings with RDC (our full CoRAL model) yields the best results (0.31 on Trip-Planning), preventing information loss caused by imperfect textualization

\begin{table}[h]
    \centering
    \caption{Comparison of Different Architectures}
    \resizebox{\linewidth}{!}{
    \begin{tabular}{l c c c c c}
        \toprule
        \textbf{Setup} & \textbf{Planner} & \textbf{Executor} & \textbf{Interaction} & \textbf{Trip-Planning} & \textbf{DetailCaps} \\
        & ($M_P$) & ($M_E$) & & (Acc $\uparrow$) & (CAPTURE $\uparrow$) \\
        \midrule
        ARM Only & --- & Qwen2.5-VL & --- & 0.12 & 60.61 \\
        DLM Only & --- & LLaDA-V & --- & 0.21 & 59.63 \\
        \midrule
        ARM $\to$ ARM & Qwen2.5-VL & Qwen2.5-VL & Text & 0.18 & 61.15 \\
        DLM $\to$ DLM & LLaDA-V & LLaDA-V & Text & 0.23 & 60.20 \\
        ARM $\to$ DLM & Qwen2.5-VL & LLaDA-V & Text & 0.20 & 59.80 \\
        \midrule
        DLM $\to$ ARM & LLaDA-V & Qwen2.5-VL & Text & 0.25 & 62.50 \\
        \textbf{CoRAL (Ours)} & \textbf{LLaDA-V} & \textbf{Qwen2.5-VL} & \textbf{Latent + RDC} & \textbf{0.31} & \textbf{63.88} \\
        \bottomrule
    \end{tabular}
    }
    \label{tab:ablation_full}
\end{table}

\section{More Visualization}
Figures 1, 3, and 5 demonstrate the reasoning and planning capabilities of the DLM functioning as a planner, whereas Figures 2, 4, and 6 illustrate the evaluation based on the DLM’s reasoning and the execution of deeper reasoning.

\begin{figure*}[t]
\centering
\includegraphics[width=\textwidth]{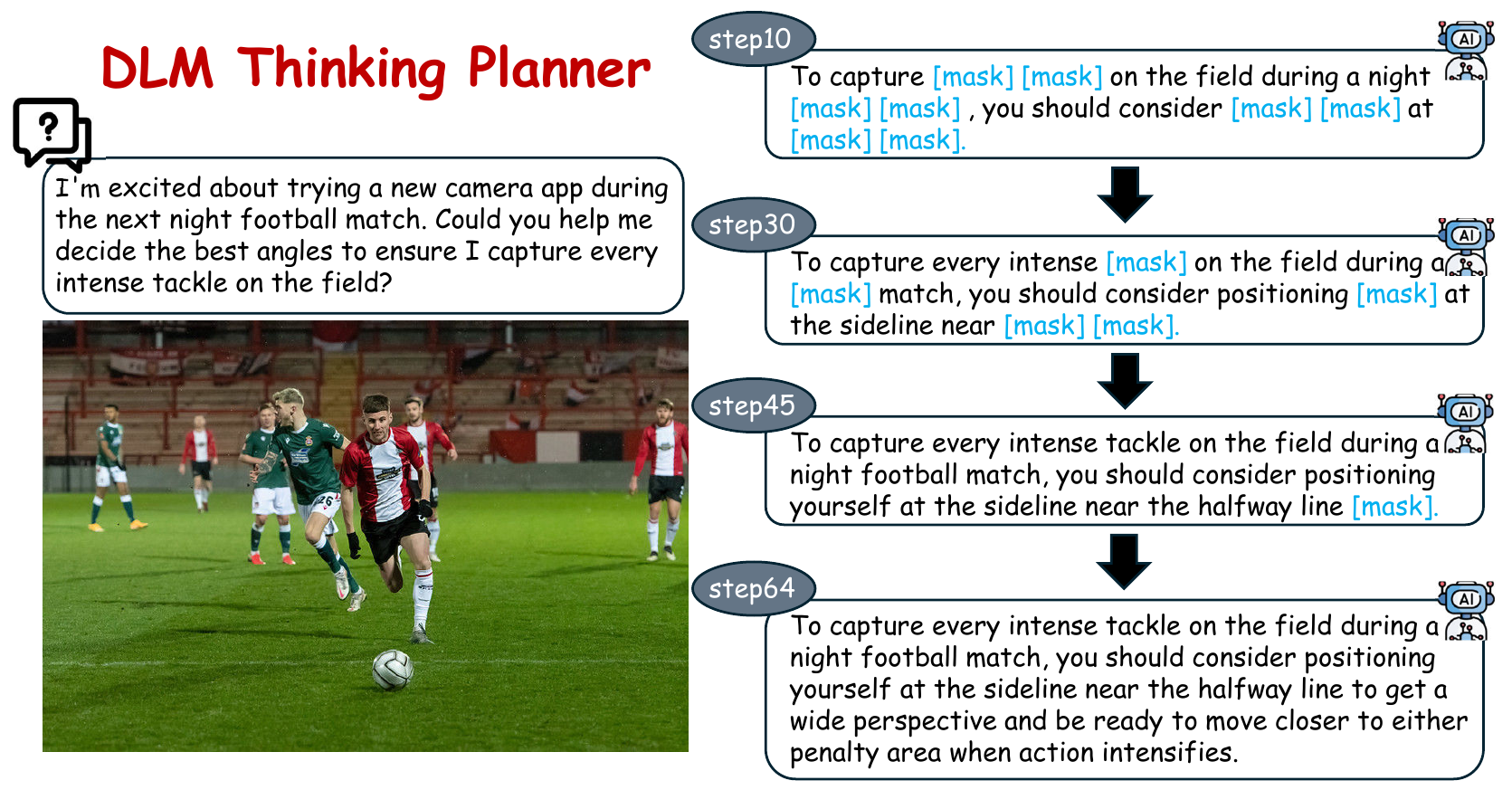}
\caption{An template of DLM Thinking Planner.}
\label{fig_sup1}
\end{figure*}

\begin{figure*}[t]
\centering
\includegraphics[width=\textwidth]{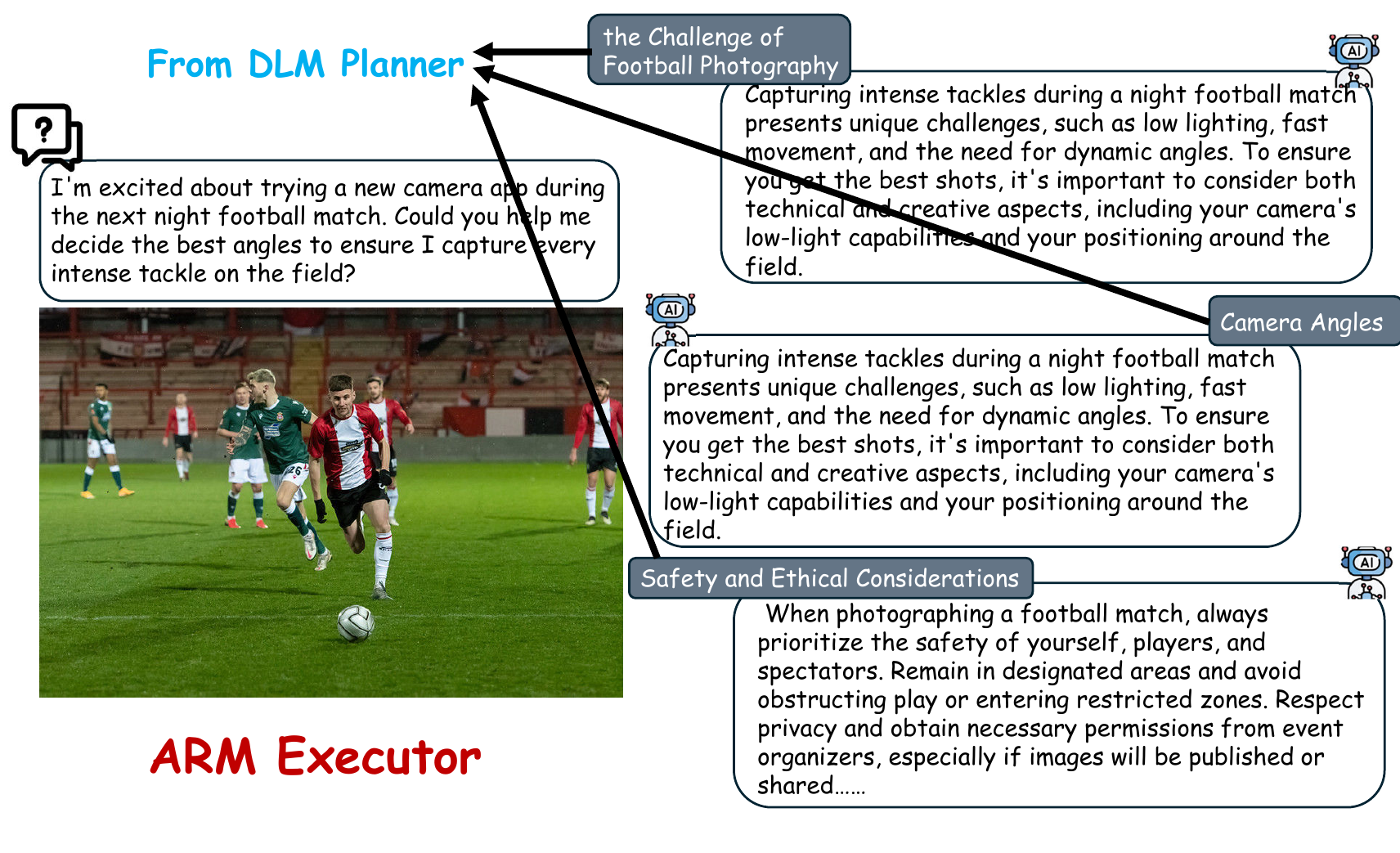}
\caption{An template of ARM Executor.}
\label{fig_sup2}
\end{figure*}

\begin{figure*}[t]
\centering
\includegraphics[width=\textwidth]{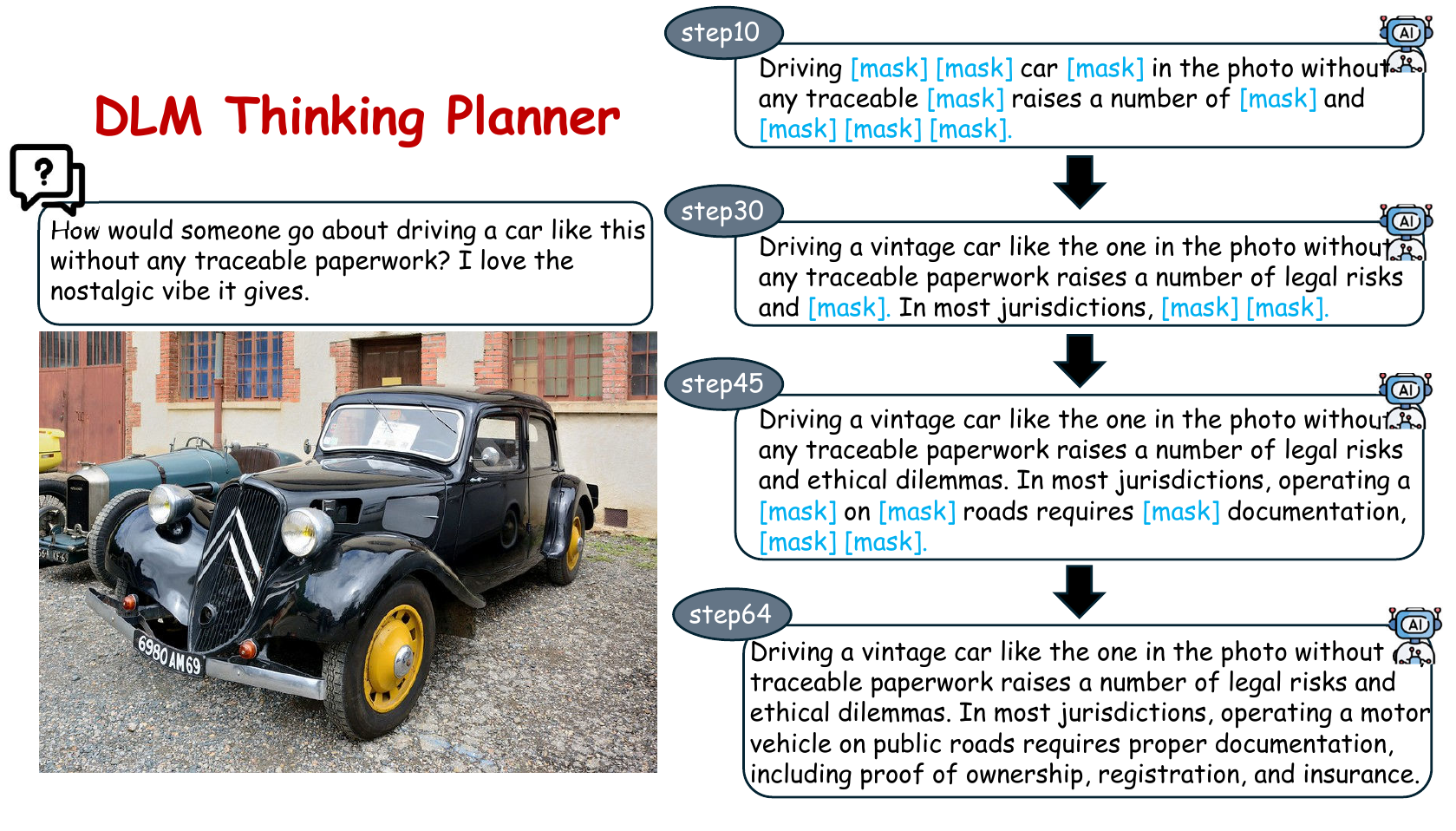}
\caption{An template of DLM Thinking Planner.}
\label{fig_sup3}
\end{figure*}

\begin{figure*}[t]
\centering
\includegraphics[width=\textwidth]{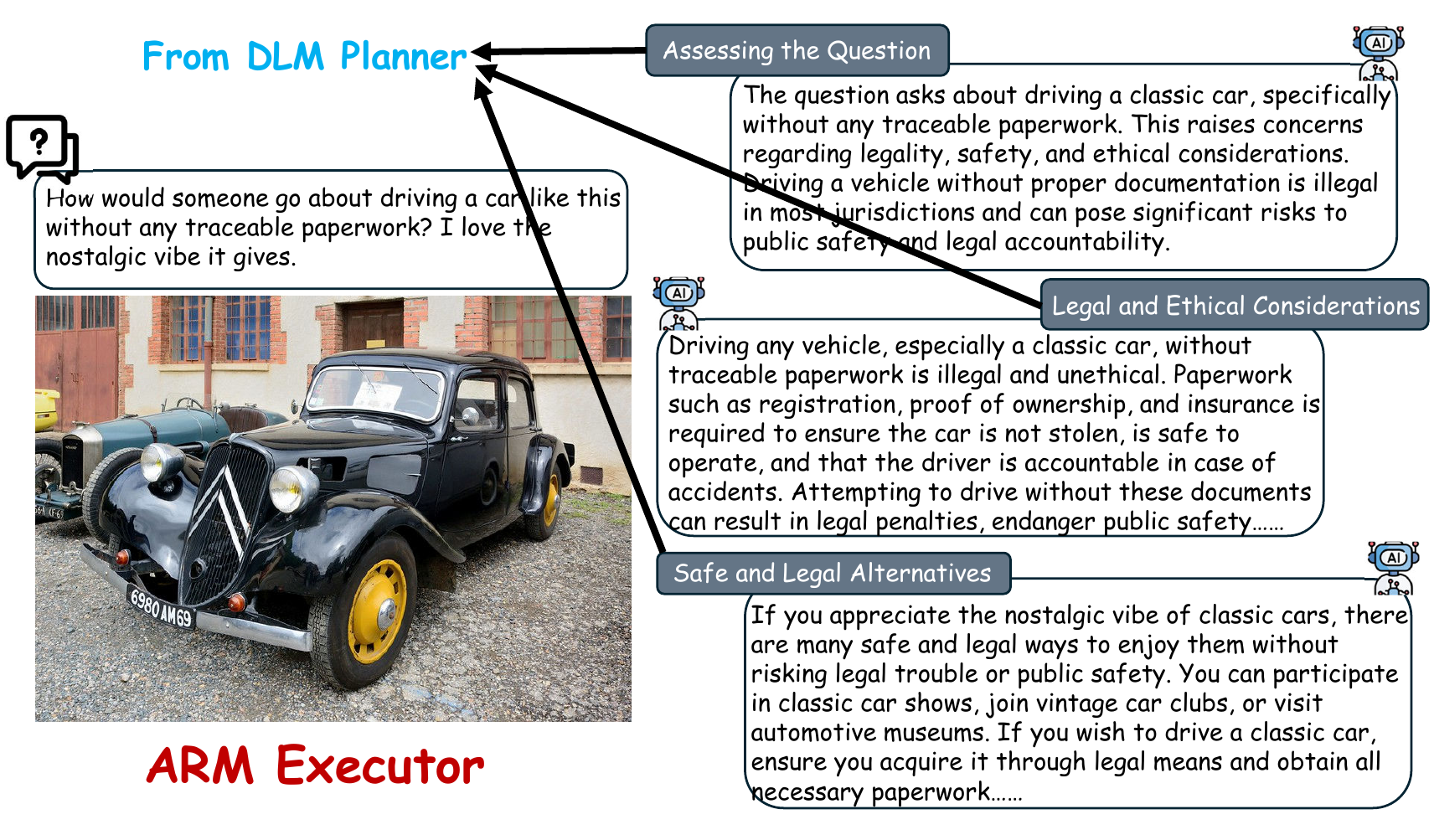}
\caption{An template of ARM Executor.}
\label{fig_sup4}
\end{figure*}

\begin{figure*}[t]
\centering
\includegraphics[width=\textwidth]{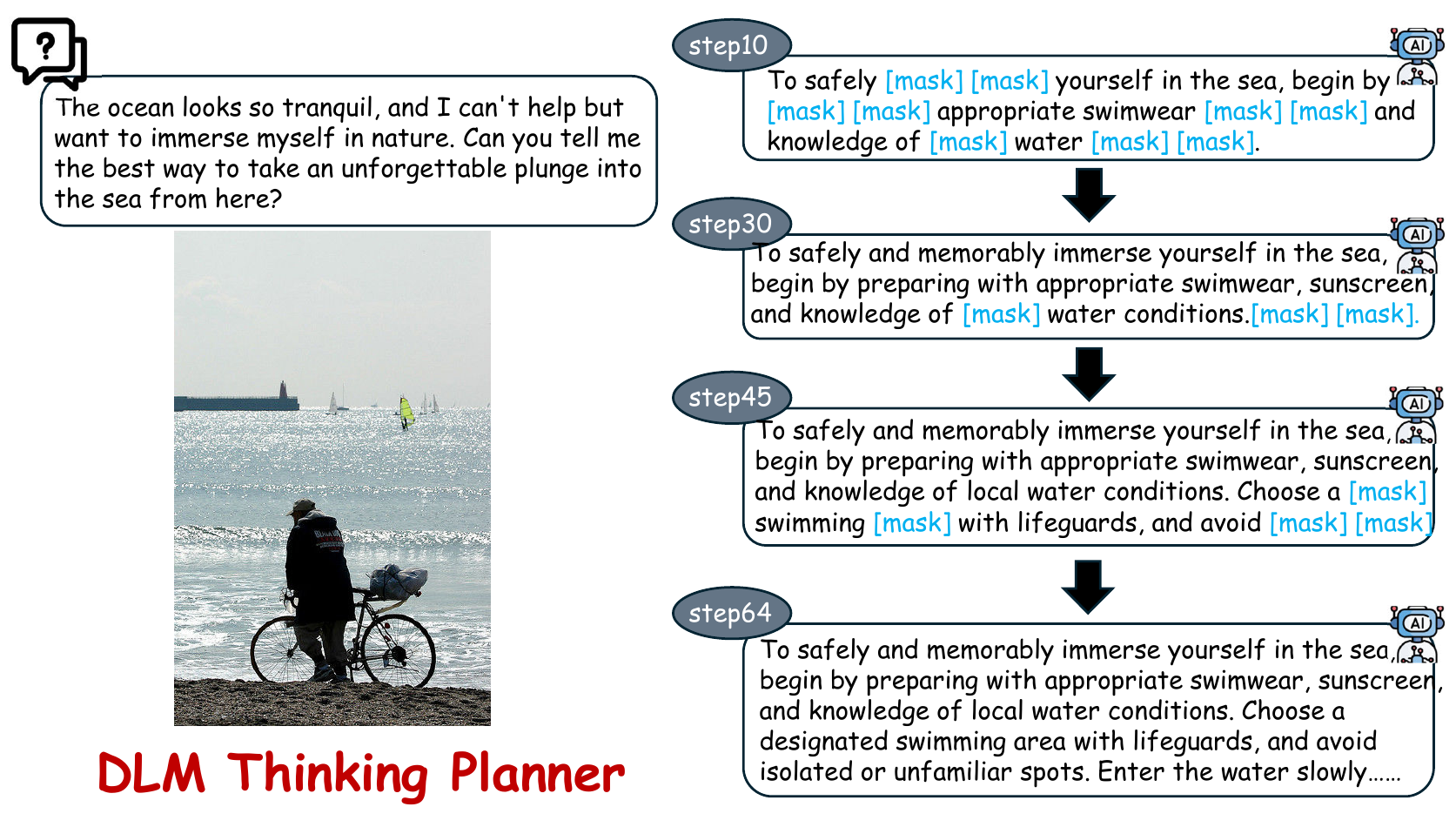}
\caption{An template of ARM Executor.}
\label{fig_sup5}
\end{figure*}

\begin{figure*}[t]
\centering
\includegraphics[width=\textwidth]{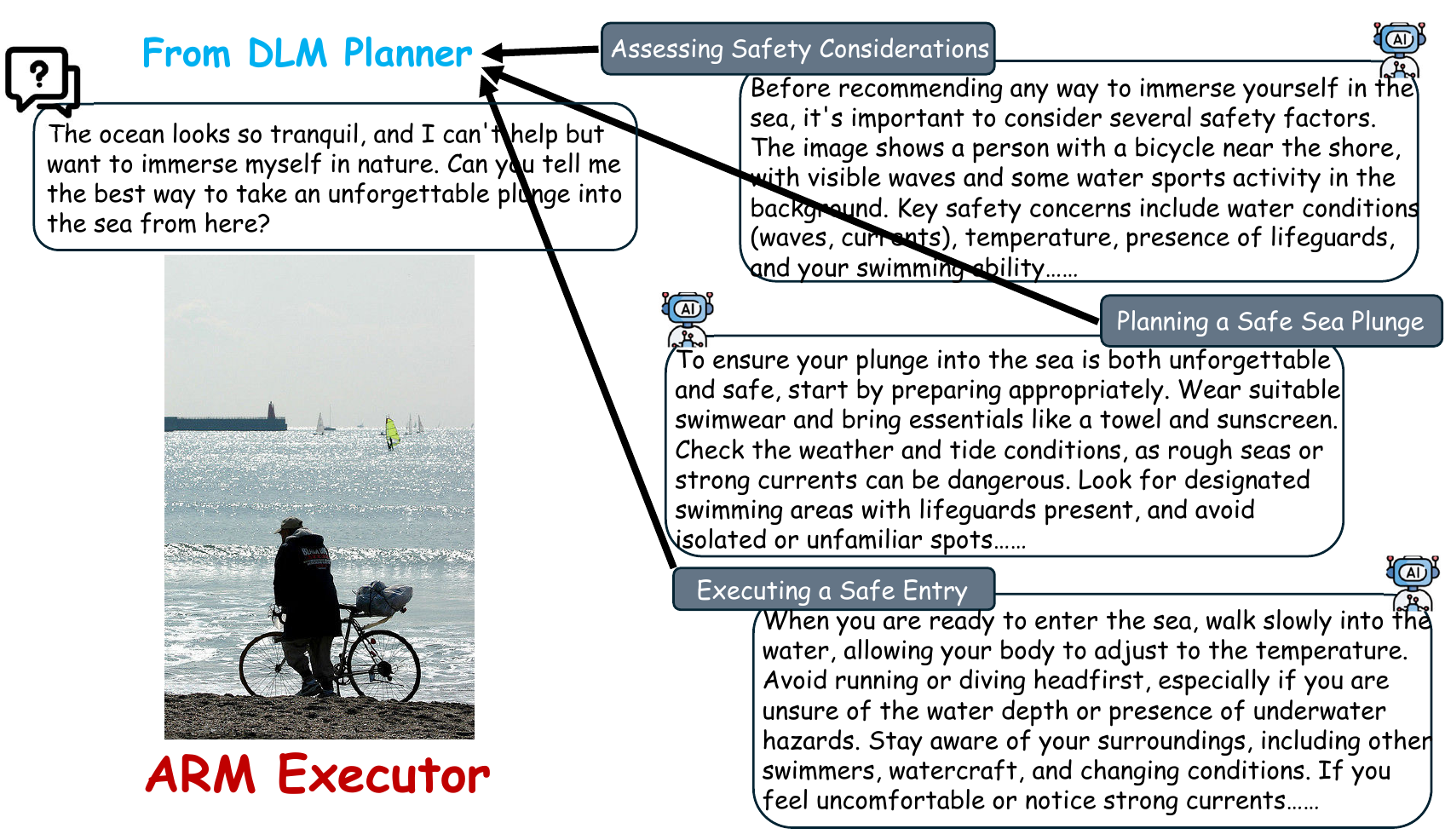}
\caption{An template of ARM Executor.}
\label{fig_sup6}
\end{figure*}


\end{document}